\documentclass[11pt]{article}

\usepackage[margin=1in]{geometry}
\usepackage{smile}

\usepackage{amscd,amssymb}
\usepackage{enumitem}
\usepackage{mathtools}
\usepackage{amsmath}
\usepackage{graphicx}
\usepackage[utf8]{inputenc}
\usepackage[export]{adjustbox}
\usepackage{wrapfig}
\usepackage{amstext}
\usepackage{pstricks,pst-node,pst-plot,pst-coil}
\usepackage{amsthm}
\usepackage{enumitem}
\usepackage{amsmath}
\usepackage{color}
\usepackage{mathtools}

\usepackage[colorlinks, linkcolor=blue, anchorcolor=blue, citecolor=blue,hypertexnames=false]{hyperref}
\usepackage[english]{babel}
\usepackage{natbib}
\usepackage{booktabs}
\usepackage{mathrsfs}
\usepackage{comment}
\usepackage{float}

\newcommand{\ReLU}{\mathrm{ReLU}}

\newcommand{\proj}{\mathrm{Proj}}

\usepackage{tikz}
\usepackage{pifont}%
\usepackage{lineno}

\usepackage{diagbox}
\usetikzlibrary{decorations.pathreplacing,calligraphy}
\usepackage{pgfplots}
\pgfplotsset{compat=1.11}
\usepgfplotslibrary{groupplots}

\allowdisplaybreaks

\providecommand{\keywords}[1]{\textbf{Key words:} #1}

\title{ Coefficient-to-Basis Network: A Fine-Tunable Operator Learning Framework for Inverse Problems with Adaptive Discretizations and Theoretical Guarantees }

\date{}

\author{Zecheng Zhang \thanks{Department of Mathematics, Florida State University, Tallahassee, FL 32306.  \texttt{Email: zecheng.zhang.math@gmail.com.} }
\and
Hao Liu \thanks{Department of Mathematics, Hong Kong Baptist University, Hong Kong, China. \texttt{Email: haoliu@hkbu.edu.hk.} }
\and
Wenjing Liao \thanks{School of Mathematics, Georgia Institute of Technology, Atlanta, GA 30332. \texttt{Email: wliao60@gatech.edu.} }
\and
Guang Lin \thanks{Department of Mathematics and Mechanical Engineering, Purdue University, West Lafayette, IN 47907. \texttt{Email: guanglin@purdue.edu.}}
}

\begin{document}

\maketitle

\begin{abstract}
We propose a Coefficient-to-Basis Network (C2BNet), a novel framework for solving inverse problems within the operator learning paradigm. C2BNet efficiently adapts to different discretizations through fine-tuning, using a pre-trained model to significantly reduce computational cost while maintaining high accuracy. Unlike traditional approaches that require retraining from scratch for new discretizations, our method enables seamless adaptation without sacrificing predictive performance. 
Furthermore, we establish theoretical approximation and generalization error bounds for C2BNet by exploiting low-dimensional structures in the underlying datasets. 
Our analysis demonstrates that C2BNet adapts to low-dimensional structures without relying on explicit encoding mechanisms, highlighting its robustness and efficiency. To validate our theoretical findings, we conducted extensive numerical experiments that showcase the superior performance of C2BNet on several inverse problems. The results confirm that C2BNet effectively balances computational efficiency and accuracy, making it a promising tool to solve inverse problems in scientific computing and engineering applications.  
\end{abstract}


\section{Introduction}
Operator learning between infinite-dimensional function spaces is an important task that arises in many disciplines of science and engineering. 
In recent years, deep neural networks have been successfully applied to learn operators for solving numerical partial differential equations (PDEs) \citep{bhattacharya2021model,lu2021learning,li2020fourier}, image processing \citep{ronneberger2015u}, and inverse problems \citep{li2020nett,fan2019solving}.

Operator learning is challenging in general, since the input and output functions lie in infinite-dimensional spaces. To address this difficulty, many deep operator learning approaches are proposed in an encoder-decoder framework. This framework employs encoders to map the input and output functions to finite-dimensional vectors and then learn a map between these dimension-reduced vectors. Popular deep operator learning methods in this encoder-decoder framework  include  PCANet with principal component analysis \citep{hesthaven2018non,bhattacharya2021model,de2022cost} and Fourier Neural Operators (FNO) based on fast Fourier transforms \citep{li2022fourier, wen2022u, lifourier,guibas2111adaptive, zhu2023fourier}. 
These methods utilize deterministic or data-driven linear encoders and decoders for dimension reduction, and neural networks are used to learn the map. For nonlinear dimension reduction, autoencoders have demonstrated success in extracting low-dimensional nonlinear  structures in data \citep{bourlard1988auto,hinton1993autoencoders,schonsheck2022semi,liu2024deep}, and autoencoders have been applied to operator learning in \citet{seidman2022nomad,kontolati2023learning,liu2025generalization}.
Other widely used deep operator learning architectures include the Deep Operator Network (DeepONet) \citep{lu2021learning, lin2023b, lin2021accelerated, zhang2024d2no, zhang2023belnet, zhang2024modno, goswami2022physics, wang2021learning, yu2024nonlocal,hao2025laplacian}, random feature models \citep{nelsen2024operator}, among others.

Apart from the computational advances, theoretical works were  established to understand the representation and generalization capabilities of deep operator learning methods. A theoretical foundation on the approximation property of PCANet was established in \citet{bhattacharya2021model}. This was followed by a more comprehensive study in \cite{lanthaler2023operator}, which derived both upper and lower bounds for such approximations. Further contributions were made in \cite{liu2024deep}, where a generalization error bound was established for the encoder-decoder-based neural networks, including PCANet as a specific instance.
The universal approximation property of FNO was analyzed in \cite{kovachki2021universal}.
Deep Operator Networks (DeepONets) were proposed based on a universal approximation theory in  \cite{chen1995universal}, and were further analyzed in \cite{lanthaler2022error,schwab2023deep}. More recently, neural scaling laws governing DeepONets were investigated in \cite{liu2024neural}  based on  approximation and generalization theories. Furthermore, \cite{lanthaler2023parametric} explored lower bounds on the parameter complexity of operator learning, demonstrating that achieving a power scaling law for general Lipschitz operators is theoretically impossible, despite of the empirical observations in \citet{lu2021learning,de2022cost}.
However, these theoretical frameworks remain limited to fully explain the empirical success of deep neural networks in operator learning, primarily due to the inherent challenges posed by the curse of dimensionality.

In science and engineering applications, many datasets exhibit low-dimensional structures. For example, even though the images in ImageNet \citep{russakovsky2015imagenet} have an ambient dimension exceeding $150,000$,  \cite{popeintrinsic} showed that the ImageNet dataset has an intrinsic dimension about $40$. Many high-dimensional datasets exhibit repetitive patterns and special structures including rotation and translation, which contribute to a low intrinsic dimensionality \citep{tenenbaum2000global,osher2017low}.
In PDE-related inverse problems, one often needs to infer certain unknown parameters in the PDE from given observations, such as inferring the permeability in  porous media equations given the solutions \citep{efendiev2006preconditioning}.
In inverse problems, the unknown parameters often exhibit a low-dimensional structure. For example, the unknown parameters considered in \citet{lu2021deep,hasani2024generating} are generated by a few Fourier bases.

By leveraging low-dimensional structures, existing deep learning theory has shown that the approximation and generalization errors of deep neural networks for function estimation converge at a fast rate depending on the intrinsic dimension of data or learning tasks \citep{chen2022nonparametric,nakada2020adaptive,liu2021besov,cloninger2021deep,schmidt2020nonparametric}, in contrast to a slower rate of convergence in high-dimensional spaces \citep{yarotsky2017error,lu2021deep,suzukiadaptivity}. 
A generalization error bound of an autoencoder-based neural network (AENet) for operator learning  was established in \cite{liu2025generalization}, where the convergence rate depends on the intrinsic dimension of the dataset. Recently, generalization error bounds  were established in \cite{dahal2022deep,havrillaunderstanding} for generative models and transformer neural networks when the input data lie on a low-dimensional manifold.
Another approach to mitigate the curse of dimensionality considers operators with special structures, such as those arising in elliptic partial differential equations \citep{marcati2023exponential} or holomorphic operators \citep{opschoor2022exponential,adcock2024optimal}.

In this paper, we address operator learning problems arising from inverse problems \citep{wu2019inversionnet, zhu2023fourier} to determine inverse Quantities of Interest (QoIs) associated with PDEs given observations on PDE solutions. In many practical scenarios, the inverse QoIs in PDEs exhibit a simpler structure compared to the corresponding PDE solutions.
For instance, in the context of heat diffusion in non-uniform materials, the diffusivity parameter is inherently material dependent. 
If the domain can be partitioned into several subregions, each characterized nearly by a uniform material property, the diffusivity parameter can be represented by a piecewise constant function \citep{liu2021rate}. A similar framework is explored in \cite{vaidya1996convection} for the convection-diffusion of solutes in heterogeneous media. When the unknown parameter in PDEs is piecewise constant, it can be expressed by a linear combination of characteristic functions, each corresponding to a distinct subregion.
These characteristic functions form a set of orthogonal bases. In contrast, the solutions to PDEs often exhibit more complex structures. Consequently, it is natural to assume that the PDE parameters possess low-dimensional linear structures, while the PDE solutions exhibit low-dimensional nonlinear structures.

Motivated by this observation, we propose a novel framework termed a Coefficient-to-Basis Network (C2BNet) to solve inverse problems. 
The proposed network architecture comprises two key components: (1) a coefficient network, which maps the input function (representing the PDE solution) to the coefficients corresponding to the basis functions of the output (PDE parameter), and (2) a basis network, which is responsible for learning the basis functions in the ouput space.

A significant challenge in solving inverse problems by machine learning lies in adapting to new tasks or inferring QoIs in different regions of the domain based on new data. A common scenario arises when the new QoIs are discretized on a finer mesh, leading to a higher-dimensional output compared to the previous tasks.
To address this challenge, we propose a fine-tuning approach for a pretrained network that was initially trained on coarser discretizations. Our theoretical analysis demonstrates that only updating  a linear layer of the pre-trained network is sufficient to adapt to the new task on finer discretizations. 
This approach significantly reduces the computational complexity associated with retraining the entire network on the new dataset while maintaining the accuracy of the prediction.
By assuming that the input functions reside on a low-dimensional manifold, we establish approximation and generalization error bounds for our proposed C2BNet. In particular, these error bounds depend crucially on the intrinsic dimension of the input functions. Our contributions are summarized as follows:
\begin{itemize}
\item We introduce C2BNet, a novel framework for learning operators designed for PDE-based inverse problems. C2BNet can be efficiently fine-tuned to accommodate different discretizations using a pre-trained network, significantly reducing the computational cost of retraining without sacrificing accuracy.

\item We establish approximation and generalization error bounds for C2BNet incorporating low-dimensional structures in datasets. Our results demonstrate that C2BNet adapts to low-dimensional data structures without an additional encoder-decoder mechanism.

\item We validate the performance of C2BNet and the theoretical findings through comprehensive numerical experiments.
\end{itemize}

The remainder of this paper is structured as follows. In Section \ref{sec.preliminary}, we introduce some concepts and background definitions necessary for an understanding of the proposed framework. Section \ref{sec.methodology} presents details of of the proposed C2BNet. A theoretical analysis of C2BNet is provided in Section \ref{sec.theory}, with proofs detailed in Section \ref{sec.mainproof}. The efficacy of C2BNet is demonstrated through a series of numerical experiments in Section \ref{sec.experiments}. Finally, we conclude the paper in Section \ref{sec.conclusion}.

{\bf Notations}: In this paper, we use bold letters to denote vectors and normal letters to denote scalars. Calligraphic letters are used to denote sets. For a set $\Omega$, we use $|\Omega|$ to denote its volume.

\section{Preliminary}
\label{sec.preliminary}
We first introduce some definitions about manifolds. We refer readers to \citet{tu2011manifolds,lee2006riemannian} for more detailed discussions.
\begin{definition}[Chart]
    Let $\cM$ be a $d$-dimensional manifold embedded in $\RR^D$. A chart of $\cM$ is a pair $(Q,\phi)$ where $Q\subset \cM$ is an open subset of $\cM$, $\phi: Q\rightarrow \RR^d$ is a homeomorphism.
\end{definition}
The transformation $\phi$ in a chart defines a coordinate map on $Q$. 
An atlas of $\cM$ is a collection of charts that covers $\cM$:
\begin{definition}[$C^k$ Atlas]
Let $\{(Q_k,\phi_k)\}_{k\in \cK}$ be a collection of charts of $\cM$ with $\cK$ denoting the set of indices. It is a $C^k$ atlas of $\cM$ if 
\begin{enumerate}[label=(\roman*)]
    \item $\cup_{k\in \cK} Q_k=\cM$,
    \item the mappings
    $$\phi_j\circ \phi_k^{-1}: \phi_k(Q_j\cap Q_k)\rightarrow \phi_j(Q_j\cap Q_k) \mbox{ and } \phi_k\circ \phi_j^{-1}: \phi_j(Q_j\cap Q_k)\rightarrow \phi_k(Q_j\cap Q_k)$$
    are $C^k$ for any $j,k\in \cK$.
\end{enumerate}
A finite atlas is an atlas with a finite number of charts.
\end{definition}
On a smooth manifold, we define $C^s$ functions as follows.
\begin{definition}[$C^s$ functions on $\cM$]
    Let $\cM$ be a smooth manifold and $f: \cM\rightarrow \RR$ be a function defined on $\cM$. The function is called a $C^s$ function on $\cM$ if for any chart $(Q,\phi)$ of $\cM$, the composition $f\circ \phi^{-1}: \phi(Q)\rightarrow \RR$ is a $C^s$ function.
\end{definition}
To measure the complexity of a manifold, we define reach as follows.
\begin{definition}[Reach \citep{federer1959curvature,niyogi2008finding}]
    Let $\cM$ be a manifold embedded in $\RR^D$. We define 
    $$
    G=\{\xb\in \RR^D: \exists\  \yb \neq \zb\in \cM \mbox{ such tha } d(\xb,\cM)=\|\xb-\yb\|_2=\|\xb-\zb\|_2\},
    $$
    where $d(\xb,\cM)$ is the distance from $\xb$ to $\cM$. The reach of $\cM$ is defined as
    $$
    \tau=\inf_{\xb\in \cM} \inf_{\yb\in G} \|\xb-\yb\|_2.
    $$
\end{definition}
The reach of a manifold gives a characterization of curvature. 
A hyper-plane has a reach $\tau=\infty$, and a hyper-sphere with radius $r$ has a reach $\tau=r$.

In this paper, we consider feedforward neural networks in the form of
\begin{align}
	f_{\rm NN}(\xb)=W_L\cdot \ReLU(W_{L-1}\cdots \ReLU(W_1\xb+\bb_1) \cdots +\bb_{L-1})+\bb_L,
	\label{eq.relu.net}
\end{align}
where $W_l$'s are weight matrices, $\bb_l$'s are bias, and $\ReLU(a)=\max\{a,0\}$ is the rectified linear unit and is applied elementwisely to its argument.

We consider the following network class
\begin{align}
	\cF_{\rm NN}(&d_1,d_2,L,p,K,\kappa,R)= \nonumber\\
	&\{\fb_{\rm NN}=[f_1,\cdots, f_{d_2}]^{\top}|  f_{k}:\Omega \rightarrow \RR \mbox{ is in the form of (\ref{eq.relu.net}) with } L \mbox{ layers, width bounded by } p, \nonumber\\
	& \|f_{k}\|_{L^{\infty}(\Omega)}\leq R, \ \|W_l\|_{\infty,\infty}\leq \kappa, \ \|b_l\|_{\infty}\leq \kappa,\  \sum_{l=1}^L \|W_l\|_0+\|b_l\|_0\leq K, \ \forall l   \}.
\end{align}

\section{Coefficient to Basis Network (C2BNet) for operator learning}
 \label{sec.architecture}
 Our objective is to learn an unknown operator \begin{equation}
 \Psi: \cX\rightarrow \cY,
 \label{eq:operator}
 \end{equation}
 where $\cX\subset L^2(\Omega_{\cX})$ and $\cY\subset L^2(\Omega_{\cY})$ are input and output function sets with domain $\Omega_{\cX}\subset \RR^{s_1}$ and $\Omega_{\cY}\subset \RR^{s_2}$ respectively. In addition, $\cX$ and $\cY$ belong to separable Hilbert spaces with inner products $\langle u_1,u_2\rangle_{\cX}$ and $\langle v_1,v_2\rangle_{\cY}$, respectively.

 In PDE-based inverse problems, the input set $\cX$ contains the PDE solutions and the output set $\cY$ contains the PDE parameters to be determined. Motivated by low-dimensional linear structures in the PDE parameters
\citep{liu2021rate,vaidya1996convection}, we assume that the output functions in $\cY$ approximately lie in a low-dimensional subspace.
\begin{assumption}\label{assum.cY}
	Suppose there exists a set of orthonormal functions $\{\omega_k\}_{k=1}^{d_2}$ and a constant $\zeta\geq 0$ so that  any $v\in \cY$ satisfies
    \begin{align}
        \left\|v-\sum_{k=1}^{d_2} \langle v,\omega_k\rangle_{\cY}\omega_k\right\|_{L^{\infty}(\Omega_{\cY})}\leq \zeta.
    \end{align}
    We denote 
	\begin{align}
		\proj(v)=\sum_{k=1}^{d_2} \alpha^v_k(v)\omega_k, \quad \mbox{ with } \alpha^v_k(v)=\langle v,\omega_k\rangle_{\cY}.
		\label{eq.v.decom}
	\end{align}
\end{assumption}
Assumption \ref{assum.cY} ensures the existence of a set of orthonormal bases whose span can approximate functions in $\cY$ with  $\zeta$ error. This assumption is inspired by finite element methods and existing works on operator learning. In finite element methods \citep{li2017numerical}, the output function space is approximated by a finite element space, spanned by a set of basis functions. For the encoder-decoder-based operator learning approaches, such as PCANet and Fourier neural operators \citep{bhattacharya2021model,lifourier,liu2024deep,lanthaler2023operator}, linear encoders and decoders are used based on certain bases. In these works, the bases are either deterministic, such as Fourier bases and Legendre polynomials, or estimated by data-driven tools, such as PCA. 

For any input $u\in \cX$ and output $v\in \cY$, we denote their discretized counterparts by $\ub=S_{\cX}(u)\in \RR^{D_1}$ and $\vb=S_{\cY}(v)\in \RR^{D_2}$ respectively, where $S_{\cX}$ and $S_{\cY}$ denote the corresponding discretization operators.
 For a given discretization operator $S_{\cX}$, we denote the  induced inner product in $\RR^{D_1}$ by $\langle \cdot,\cdot \rangle_{S_{\cX}}$. One way to define such an inner produce is by using a quadrature rule:
\begin{align}
    \langle S_{\cX}(u_1), S_{\cX}(u_2) \rangle_{S_{\cX}}= \sum_{k=1}^{D_1} \tau_k (\ub_1)_k(\ub_2)_k,
    \label{eq.inner.dis}
\end{align}
where $\tau_k>0$ are the weights in the quadrature rule. When {the functions in $\cX$ are smooth} and the discretization grid of $S_{\cX}$ is sufficiently fine, we expect $\|S_{\cX}(u)\|_{S_{\cX}}\approx \|u\|_{\cX}$ for any $u\in \cX$.
We make the following assumption on the discretization operator $S_{\cY}$:
\begin{assumption}
\label{assum.norm}
    Assume that, for any $v\in \cY$, 
    \begin{align}
        0.5\|v\|_{\cY} \leq \|S_{\cY}(v)\|_{S_{\cY}}\leq 2 \|v\|_{\cY}, \quad 0.5\|\proj(v)\|_{\cY} \leq \|S_{\cY}\circ\proj(v)\|_{S_{\cY}}\leq 2 \|\proj(v)\|_{\cY}.
    \end{align}
\end{assumption}
Assumption \ref{assum.norm} is a weak assumption. This holds as long as {functions in $\cY$ are uniformly regular and} the discretization grid is sufficiently fine. For example, according to the Nyquist–Shannon sampling theorem \citep{shannon1949communication}, bandlimited functions  can be completely determined from its discretized counterparts on a sufficiently fine grid. A specific example is provided in \citet[Example 1]{liu2025generalization} to demonstrate that Assumption \ref{assum.norm} can be satisfied if the functions in $\cY$ are band limited and the discretization grid is sufficiently fine.
We can find weights $\omega_k$'s in (\ref{eq.inner.dis}) using a quadrature rule to approximate the continuous inner product by the discretized counterpart.

 \subsection{Coefficient to Basis Network (C2BNet) architecture}
Our network structure is designed based on the decomposition in (\ref{eq.v.decom}), which consists of two components: the coefficients $\{\alpha^v_k(v)\}_{k=1}^{d_2}$ and the bases $\{\omega_k\}_{k=1}^{d_2}$. We will design a network for each component. 
In (\ref{eq.v.decom}), for each $k$, the coefficient $\alpha^v_k(v)$ is a functional of  the output function $v$ in $\cY$. Since $v=\Psi(u)$, we have
\begin{align}
    \alpha^v_k(v)=\alpha^v_k\circ\Psi(u).
    \label{eq.ceof}
\end{align}
which implies that each coefficient is a functional of the input function $u$ as well. We construct a coefficient network $\fb_{coef}: \RR^{D_1}\rightarrow \RR^{d_2}$ to learn the mapping in \eqref{eq.ceof}:
$$\ub\xrightarrow[]{\fb_{coef}} [\alpha^v_1\circ \Psi(u),..., \alpha^v_{d_2}\circ \Psi(u)]^{\top},$$
where $\ub=S_{\cX}(u)$ denotes the discretized counterpart of $u$.

Given a disretization grid in $\Omega_{\cY}$ associated with the discretization operator $S_{\cY}$, we can represent the basis functions
 $\{\omega_k\}_{k=1}^{d_2}$ by a set of vectors $\{S_{\cY}(\omega_k)\}_{k=1}^{d_2}$. In C2BNet, we use a linear layer to learn this set: $\fb_{basis}: \RR^{d_2}\rightarrow \RR^{D_2}$.

Our operator network $\Psi_{\rm NN}$ that approximates the target operator $\Psi$ in \eqref{eq:operator} is constructed as
\begin{align}
    \Psi_{\rm NN}=\fb_{basis}\circ \fb_{coef}.
    \label{eq.psi.archi}
\end{align}
The architecture of $\Psi_{\rm NN}$ is illustrated in Figure \ref{fig:network}(a). 

\subsection{Fine-tuning on a new discretization in $\Omega_{\cY}$ }
\label{sec.methodology}
A key advantage C2BNet is adapting to different discretizations. Once the network is trained on a specific discretization associated with $ S_{\cY}$, it can be efficiently fine-tuned to accommodate a new discretization associated with $S'_{\cY}$. As indicated by (\ref{eq.ceof}), the coefficients $\{\alpha^v_k(v)\}_{k=1}^{d_2}$ are determined solely by the input function $u$ and are independent of the discretization $S_{\cY}$. In our framework, altering the discretization in \( \Omega_{\mathcal{Y}} \) only requires updating the basis network \( \fb_{basis} \) to represent \( \{S'_{\mathcal{Y}}(\omega_k)\}_{k=1}^{d_2} \) instead of \( \{S_{\mathcal{Y}}(\omega_k)\}_{k=1}^{d_2} \). Consequently, if the coefficient network \( \fb_{coef} \) accurately learns the coefficients, only \( \fb_{basis} \) needs to be fine-tuned in the new data set to learn the updated basis functions. This strategy eliminates the need to retrain the entire network, as only a single linear layer needs to be updated, which significantly reduces the computational cost.
Our fine-tuning strategy is illustrated in Figure \ref{fig:network}(b).
\begin{figure}[t!]
    \centering
    \begin{tabular}{cc}
         \includegraphics[width=0.48\textwidth]{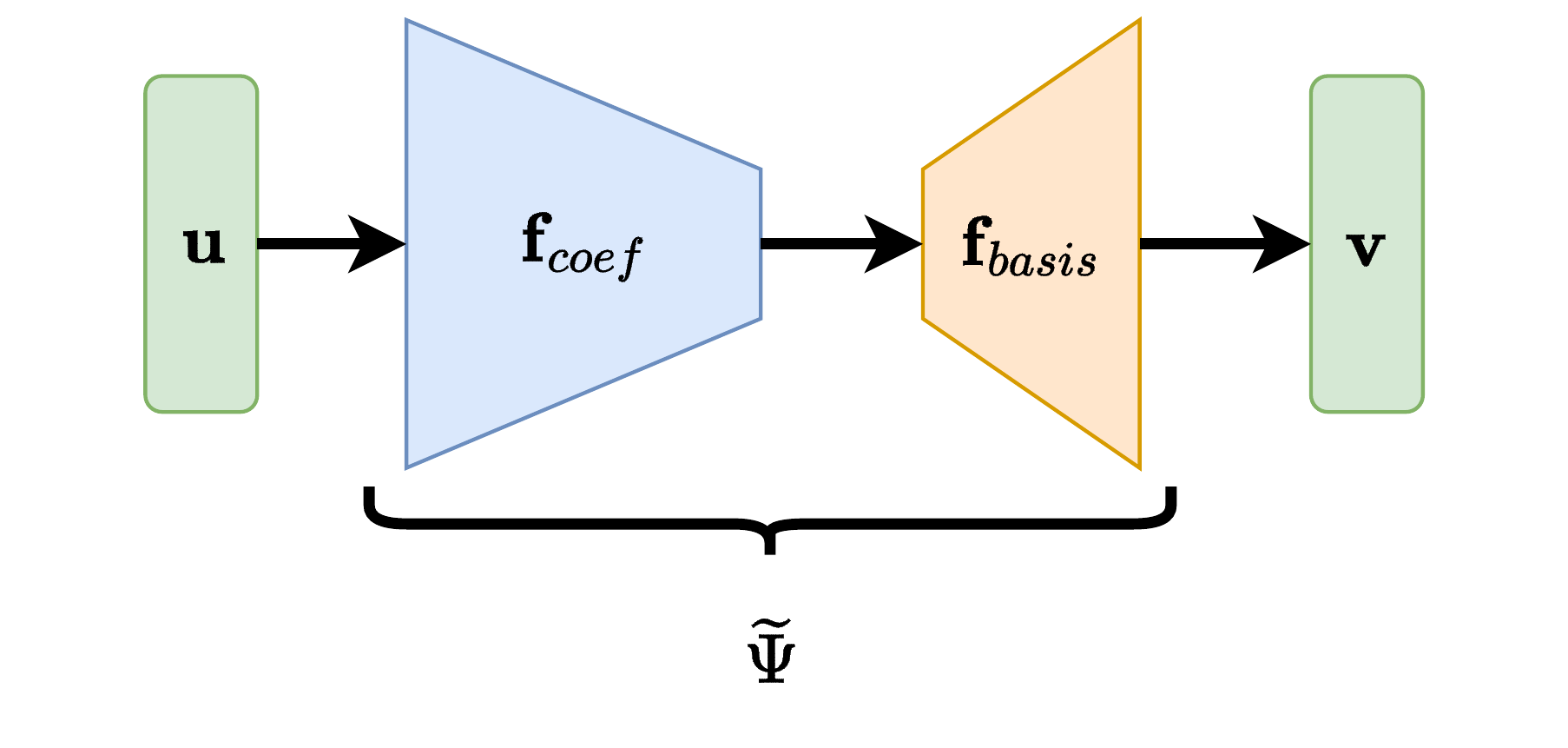}&
         \includegraphics[width=0.48\textwidth]{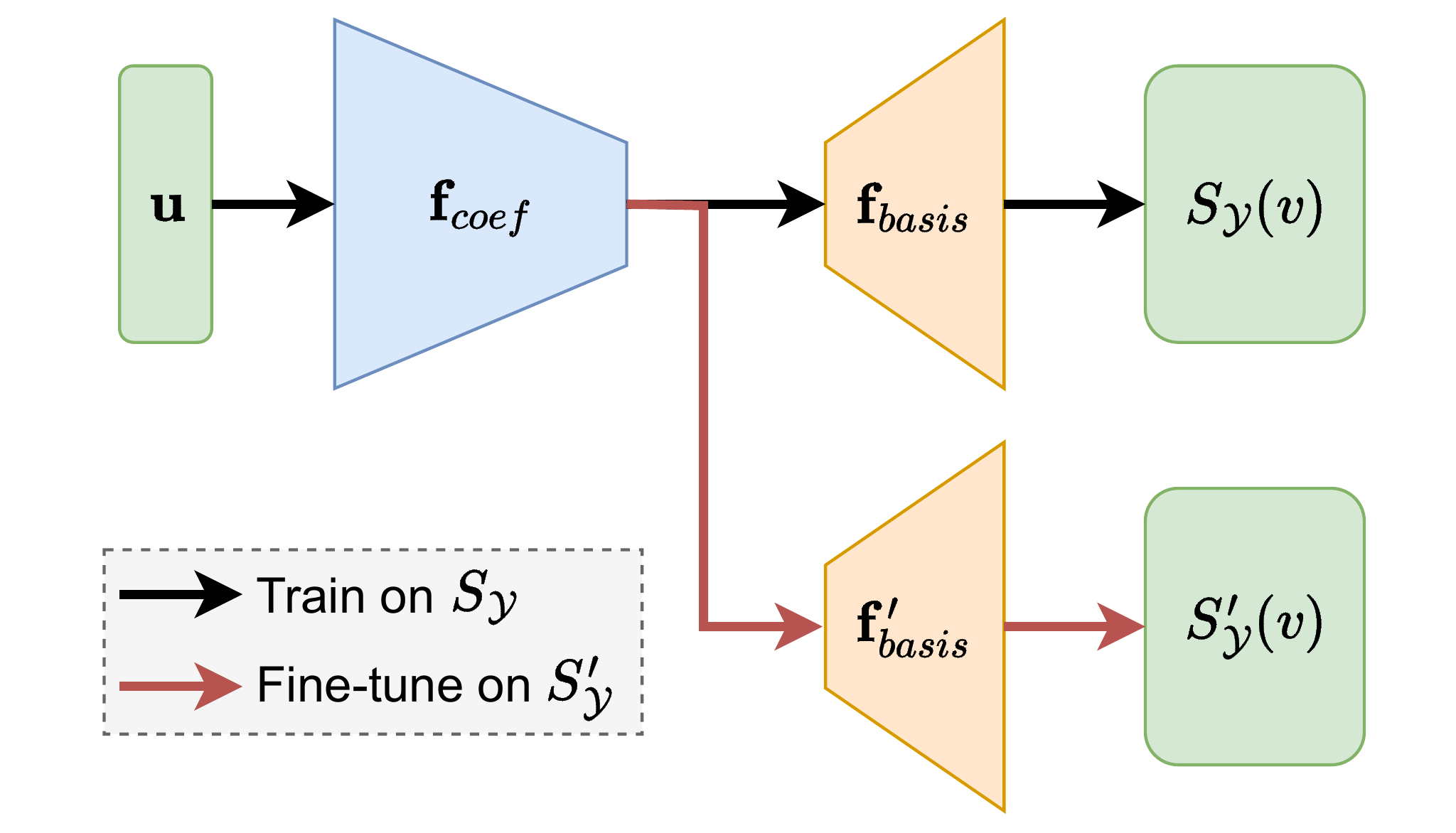}\\
         (a)& (b) 
    \end{tabular}
    \caption{(a) An illustration of C2BNet in (\ref{eq.psi.archi}). (b) Fine-tuning  from discretization $S_{\cY}$ to $S'_{\cY}$. }
    \label{fig:network}
\end{figure}

\section{Approximation and generalization theory of C2BNet}
\label{sec.theory}
In this section, approximation and generalization error bounds are established for C2BNet in Section \ref{sec.architecture}. C2BNet is designed to utilize low-dimensional linear structures in  output functions. By further leveraging low-dimensional nonlinear structures in  input functions, we prove approximation and generalization error bounds depending on the intrinsic dimension of the input functions.
All proofs are deferred to Appendix \ref{sec.mainproof}. 

Our theoretical analysis requires some assumptions.
The first assumption below says that the functions in $\cX$ and $\cY$ have a bounded domain and the function values are bounded.
\begin{assumption}
	\label{assum.boundedness}
	Suppose the followings hold for $\cX$ and $\cY$: there exist $B_1,B_2,M_1,M_2>0$ such that
	\begin{enumerate}[label=(\roman*)]
		\item For any $\xb\in \Omega_{\cX}, \yb\in \Omega_{\cY}$, $\|\xb\|_{\infty}\leq B_1$, $\|\yb\|_{\infty}\leq B_2$.
		\item For any $u\in \cX, v\in \cY$, $\|u\|_{L^{\infty}(\Omega_{\cX})}\leq M_1, \ \|v\|_{L^{\infty}(\Omega_{\cY})}\leq M_2$.
	\end{enumerate}
\end{assumption}
Assumption \ref{assum.boundedness} implies that the domain and output of any function in $\cX$ and $\cY$ are bounded. 
The next assumption leverages  low-dimensional structures in input functions:
\begin{assumption}\label{assum.cX}
	Suppose the input function set $\cX$ satisfies the followings:
	\begin{enumerate}[label=(\roman*)]
		\item Suppose $S_{\cX}$ discretizes functions in $\cX$ without loss of information: there exists a map $D_{\cX}: \RR^{D_1}\rightarrow \cX$ such that
		$$
		D_{\cX}\circ S_{\cX}(u)=u
		$$ 
		for any $u\in \cX$. 
		\item Suppose that $\{\ub =\cS_{\cX}(u): \ u \in \cX\}$ is located on a $d_1$-dimensional Riemannian manifold $\cM$ isometrically embedded in $\RR^{D_1}$, with a positive reach $\tau>0$.
		\item Suppose $D_{\cX}: \cM\rightarrow \cX$ is Lipschitz. Let $\{(Q_k,\phi_k)\}_{k \in \mathcal{K}}$ be an atlas of $\cM$. For any given chart $(U_k,\phi_k)$ in this atlas, there exists a constant $C$ such that for any $\zb_1,\zb_2\in \phi_k(U)$, 
        $$\|D_{\cX}\circ\phi_k^{-1}(\zb_1)-D_{\cX}\circ\phi_k^{-1}(\zb_2)\|_{L^{\infty}(\Omega_{\cX})}\leq C\|\zb_1-\zb_2\|_2.$$
	\end{enumerate}
	
\end{assumption}


Assumption \ref{assum.cX} assumes that $\cX$ processes a low-dimensional nonlinear structure. Assumption \ref{assum.cX}(i) imposes an one-to-one correspondence between each function in $\cX$ and their discretized counterpart.  Assumption \ref{assum.cX}(i) holds when the functions in $\cX$ are bandlimited and the discretization grid is sufficiently fine.
Assumptions \ref{assum.cX} (ii)-(iii) assumes that $\cX$ exhibits a low-dimensional non-linear structure.

Our last assumption is about the target operator $\Psi$:
\begin{assumption}\label{assum.Psi}
	Suppose that the operator $\Psi$ is Lipschitz with Lipschitz constant $L_{\Psi}$:
	$$
	\|\Psi(u_1)-\Psi(u_2)\|_{\cY}\leq L_{\Psi}\|u_1-u_2\|_{\cX}.
	$$
\end{assumption}
 Assumption \ref{assum.Psi} imposes a Lipschitz property of the target operator $\Psi$, which is common in the operator learning literature. 

\subsection{Approximation theory}
Our first theoretical result is on the approximation power of C2BNet for Lipschitz operators under  Assumptions \ref{assum.norm}-\ref{assum.Psi}.
\begin{theorem}\label{thm.approx}
	Let $B_1,B_2,M_1,M_2,L_{\Psi}>0$, and suppose Assumptions \ref{assum.norm}-\ref{assum.cX} hold. For any $\varepsilon>0$, there exists a network architecture $\cF_{coef}=\cF_{\rm NN}(D_1,d_2,L,p,K,\kappa,R)$ with
	\begin{align}
			&L=O\left(\log\frac{1}{\varepsilon}\right), \ p=O(\varepsilon^{-d_1}),\ K=O\left(\varepsilon^{-d_1}\log \frac{1}{\varepsilon} + D_1\log \frac{1}{\varepsilon} +D_1\log D_1+D_2\right), \nonumber\\
		&\kappa=O(\varepsilon^{-1}), \ R=M_2|\Omega_{\cY}|,
		\label{eq.approx.architecture}
	\end{align}
    and a linear network $\cF_{basis}=\cF_{\rm NN}(d_2,D_2,1,D_2,d_2D_2,M_2,M_2)$
such that for any operator $\Psi$ satisfying Assumption \ref{assum.Psi}, such an architecture gives rise to $\widetilde{\fb}_{coef}\in \cF_{coef}, \widetilde{\fb}_{basis}\in \cF_{basis}$ and $\widetilde{\Psi}=\widetilde{\fb}_{basis} \circ \widetilde{\fb}_{coef}$ with
\begin{align}
	\sup_{u\in \cX}\|S_{\cY}\circ\proj\circ\Psi(u)-\widetilde{\Psi}\circ S_{\cX}(u)\|_{\infty}\leq \varepsilon.
\end{align}
The constant hidden in $O$ depends on $d_1,d_2,M_1,M_2, |\Omega_{\cX}|,|\Omega_{\cY}|,L_{\Psi}, \tau$ and the surface area of $\cM$. 
\end{theorem}
Theorem \ref{thm.approx} is proved in Appendix \ref{thm.approx.proof}. Theorem \ref{thm.approx} shows that if the network architecture is properly choosen, C2BNet can universally approximate Lipschitz operators after a linear projection with an arbitrary accuracy. Furthermore, the size of the network scales with $\varepsilon$ with an exponent depending on the intrinsic dimension $d_1$ of the input functions , instead of the ambient dimension $D_1$.

Our C2BNet in Theorem \ref{thm.approx} is constructed to approximate $\proj\circ\Psi$. When the outputs of $\Psi$ approximately lie in a low-dimensional subspace as assumed in Assumption \ref{assum.cY}, we can
combine Theorem \ref{thm.approx} and Assumption \ref{assum.cY} to guarantee an approximation error of $\Psi$.
\begin{corollary}\label{coro.approx}
    Let $B_1,B_2,M_1,M_2,L_{\Psi},\zeta>0$, and suppose that Assumptions \ref{assum.cY}-\ref{assum.cX} hold. For any $\varepsilon>0$, there exists a network architecture $\cF_{coef}=\cF_{\rm NN}(D_1,d_2,L,p,K,\kappa,R)$ with
	\begin{align}
			&L=O\left(\log\frac{1}{\varepsilon}\right), \ p=O(\varepsilon^{-d_1}),\ K=O\left(\varepsilon^{-d_1}\log \frac{1}{\varepsilon} + D_1\log \frac{1}{\varepsilon} +D_1\log D_1+D_2\right), \nonumber\\
		&\kappa=O(\varepsilon^{-1}), \ R=M_2|\Omega_{\cY}|,
		\label{eq.approx.architecture}
	\end{align}
    and a linear network $\cF_{basis}=\cF_{\rm NN}(d_2,D_2,1,D_2,d_2D_2,M_2,M_2)$
such that for any operator $\Psi$ satisfying Assumption \ref{assum.Psi}, such an architecture gives rise to $\widetilde{\fb}_{coef}\in \cF_{coef}, \widetilde{\fb}_{basis}\in \cF_{basis}$ and $\widetilde\Psi=\widetilde{\fb}_{basis} \circ \widetilde{\fb}_{coef}$ with
\begin{align}
	\sup_{u\in \cX}\|S_{\cY}\circ\Psi(u)-\widetilde\Psi\circ S_{\cX}(u)\|_{\infty}\leq \zeta+\varepsilon.
\end{align}
The constant hidden in $O$ depends on $d_1,d_2,M_1,M_2, |\Omega_{\cX}|,|\Omega_{\cY}|,L_{\Psi},\tau$ and the surface area of $\cM$.
\end{corollary}
Corollary \ref{coro.approx} is proved in Appendix \ref{coro.approx.proof}. Corollary \ref{coro.approx} gives an upper bound of C2BNet for approximating the original Lipschitz operators. The bound depends on the approximation error by projection and the approximation error by networks.

\subsection{Generalization error}
Suppose we are given input-output pairs $\cS=\{u_i,\widehat{v}_i\}_{i=1}^n$ with $u_i\in \cX$ sampled from a probability distribution $\rho$ and 
$$
\widehat{v}_i=\Psi(u_i)+\xi_i,
$$
where $\xi_i$'s represent i.i.d.  noise satisfying the following assumption:
\begin{assumption}
	\label{assum.noise}
    Suppose $\xi_i$'s are i.i.d. copies of $\xi \in \cY$.
	For any $\yb\in \Omega_{\cY}$, $\xi_i(\yb)$ follows a sub-Gaussian distribution with variance proxy $\sigma^2$.
\end{assumption}
Given a network architecture $\cF_{coef}$ and $\cF_{basis}$, we define the C2BNet class
\begin{align}
    \cF=\{\fb_{basis}\circ \fb_{coef}: \fb_{coef} \in \cF_{coef}, \fb_{basis}\in \cF_{basis}\}.
    \label{eq.cF}
\end{align}
Given the training data $\{\ub_i,\widehat\vb_i\}_{i=1}^n$ with $\ub_i=S_{\cX}(u_i),\widehat{\vb}_i=S_{\cY}(\widehat{v}_i)$, we train C2BNet through the following empirical risk minimization:
\begin{align}
	\widehat{\Psi}=\argmin_{\Psi_{\rm NN}\in \cF} \frac{1}{n}\sum_{i=1}^n \|\Psi_{\rm NN}(\ub_i)-\widehat{\vb}_i\|_{S_{\cY}}^2,
	\label{eq.loss}
\end{align}

The generalization error of $\widehat{\Psi}$ satisfies the following upper bound.
\begin{theorem}\label{thm.generalization}
	Let $B_1,B_2,M_1,M_2,L_{\Psi},\zeta>0$. Suppose Assumptions \ref{assum.cY}-\ref{assum.noise} hold. 
    Set the network architecture $\cF_{coef}=\cF_{\rm NN}(D_1,d_2,L,p,K,\kappa,R),\cF_{basis}=\cF_{\rm NN}(d_2,D_2,1,D_2,d_2D_2,M_2,M_2) $ with
	\begin{align}
		&L=O(\log n), \ p=O(n^{\frac{2d_1}{2+d_1}}), \ K=O(n^{\frac{2d_1}{2+d_1}}\log n + D_1\log n + D_1\log D_1+D_2), \nonumber\\
		&\kappa=O(\log n), \ R=M_2|\Omega_{\cY}|,
		\label{eq.gene.architecture}
	\end{align}
    and consider the C2BNet class $\cF$  defined in (\ref{eq.cF}).
	The minimizer of (\ref{eq.loss}), denoted by $\widehat{\Psi}$, satisfies
	\begin{align}
		\EE_{\cS}\EE_{u\sim \rho}  \left[\frac{1}{D_2}\|\widehat{\Psi}\circ S_{\cX}(u)-S_{\cY}\circ\Psi(u)\|_2^2\right]\leq CD_1(\log D_1)n^{-\frac{2}{2+d_1}}\log^3n+ 8\zeta^2,
	\end{align}
	where $C$ is a constant. Both $C$ and the constants hidden in $O$ depend on $d_1,d_2,M_1,M_2, |\Omega_{\cX}|,|\Omega_{\cY}|, $ $ L_{\Psi}, $ $ \sigma,\tau$ and the surface area of $\cM$. 
\end{theorem}
Theorem \ref{thm.generalization} is proved in Appendix \ref{thm.generalization.proof}. The error bound in Theorem \ref{thm.generalization} has two components: the first one is the network estimation error, which converges to zero as $n$ goes to infinity; the second represents the loss of information while output functions are projected to a $d_2$-dimensional subspace. If functions in $\cY$ lie on a $d_2$-dimensional subspace, then $\zeta =0$ and the squared generalization error in Theorem \ref{thm.generalization} converges at the rate $n^{-\frac{2}{2+d_1}}$. This  rate of convergence depends  on the intrinsic dimension $d_1$ instead of the ambient dimension $D_1$. By solving (\ref{eq.loss}), C2BNet can adapt to low-dimensional structures in input functions and enjoys a fast convergence rate depending on the intrinsic dimension of the input functions, attenuating the curse of dimensionality. If $\cY$ is not exactly lying in a $d_2$-dimensional subspace, the projection error usually decays as $d_2$ increases. We can choose $d_2$ to be large enough so that $\zeta$ is small.

\section{Numerical experiments}
\label{sec.experiments}
In this section, we demonstrate the effectiveness of C2BNet through three sets of numerical experiments. 
To show the effectiveness and robustness of the methods, we use the same network structures in all examples. Specifically, the network has the size $d_{input}\rightarrow 100\rightarrow 100 \rightarrow 100 \rightarrow d_{low} \rightarrow d_{output}$, where the number indicates the width of each layer (the number of neurons of each layer), and each layer is activated by the ReLU and with bias except the last layer.
The $d_{input}, d_{output}$ and $d_{low}$ are problem-dependent and are specified in each example.
All codes and code running results will be published on Google Colab and Github after the paper is published; we can provide the link if requested by the readers.

\subsection{Radiative transfer equation}
In this example, we examine the radiative transfer equation (RTE) \citep{lai2019inverse, newton2020diffusive,li2019diffusion} with a contrast scattering parameter \(\sigma(x)\), the goal is to infer the inverse QoIs $\sigma(x)$ given the solution of the PDE. 
The RTE is a fundamental model in optical tomography and represents a classic inverse problem. The objective is to recover the scattering parameter based on given observations. The RTE is described by the following equation:  
\[
s \cdot \nabla I(x,s) = \frac{\sigma(x, \omega)}{\epsilon} \bigg( \int_{\mathcal{S}^{n-1}} I(x,s') \, ds' - I(x, s) \bigg), \quad \forall x \in D, \, s \in \mathcal{S}^{n-1}.
\]
Here, \(s\) is a vector in the unit sphere \(\mathcal{S}^{n-1}\), and \(n\) denotes the spatial dimension of the problem.  

In our experiments, we set \(n=2\), making \(\mathcal{S}^{n-1} = \mathcal{S}^1\), which corresponds to the unit circle. Additionally, we set \(\epsilon = 1\) and the domain \(D = [0, 1]^2\).  

To close the model, Dirichlet boundary conditions are imposed. Specifically, for directions entering the domain (\(s \cdot \textbf{n} < 0\)), the boundary condition is defined as \(I(x,s) = I_{\text{in}}\). This applies to the boundary subset:  
\[
\Gamma^- := \{ (x, s) \in \partial D \times \mathcal{S}^{n-1} : s \cdot \textbf{n} < 0 \},
\]
where \(\textbf{n}\) is the unit outward normal vector at \(x \in \partial D\). The condition is expressed as:  
\[
I(x,s) = I_{\text{in}}(x,s) \quad \text{for all } (x, s) \in \Gamma^-.
\]
In our examples, the top, bottom and right boundaries are assigned zero incoming boundary conditions, while the left boundary is assumed to have a non-zero flow, injecting energy or particles into the domain.

The term `contrast scattering' refers to the difference in scattering properties between specific channels and the surrounding background. Each scattering includes five channels, where the size of each channel is randomly determined by two free variables: length and width. 
Figure \ref{fig_rte_scatter} illustrates four realizations of the scattering parameter \(\sigma\), highlighting the variability in channel structures across different samples.
\begin{figure}[H]
    \centering
    \includegraphics[scale = 0.4]{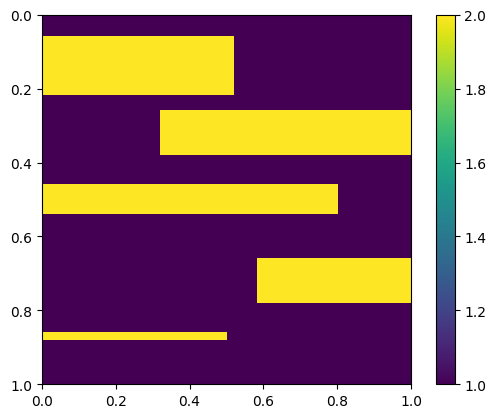}
    \includegraphics[scale = 0.4]{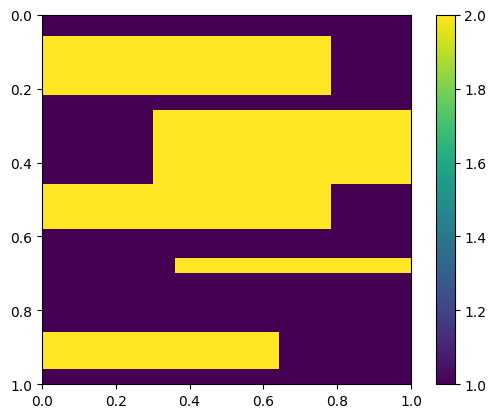}
    \caption{Demonstration of two scatter parameters for RTE. Each scattering field has 5 channels and each channels has two degrees of freedom width and height.}
    \label{fig_rte_scatter}
\end{figure}

\subsubsection{Results for RTE}
In this section, we present the numerical results. 
For input, the solution (observation) is uniformly discretized on a \(11 \times 11\) spatial mesh, with 4 velocity directions uniformly distributed on the unit circle. The inverse QoIs correspond to the target scattering parameter, which is discretized on a \(10 \times 10\) spatial mesh,  $d_{low} = 50$ for this example.
In Figure \ref{fig_rte_results}, we present the error decay in the \(L_1\) norm with respect to the number of training samples.  
From the left panel of the figure, we observe that the relative error decreases to as low as \(3\%\) when \(144\) training samples are used, demonstrating the effectiveness of the neural network. Moreover, the right panel shows the error decay in a log scale linearly, validating the power-law behavior shown in Theorem \ref{thm.generalization}.

\begin{figure}[H]
    \centering
    \includegraphics[scale = 0.4]{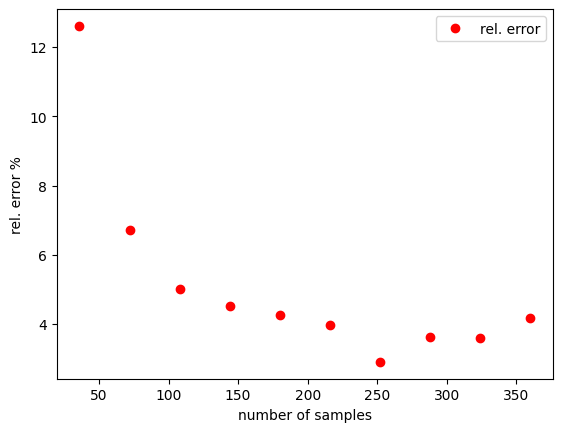}
    \includegraphics[scale = 0.4]{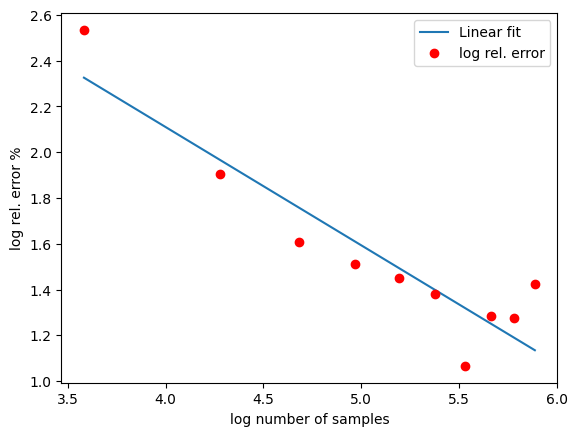}
    \caption{RTE results. Left: Relative error decay with respect to the number of training samples. Right: Error decay in log-scale. Note a linear line is fit to better demonstrate the linear decay in log.}
    \label{fig_rte_results}
\end{figure}

\subsection{Elliptic equation}
In this example, we consider the elliptic equation,
\begin{align}
    -\nabla\cdot \kappa(x, y) \nabla u = f, x\in[0, 1]^2,
\end{align}
where $\kappa(x)$ is the permeability.
The goal is to recover $\kappa(x, y)$ given the observations.
We set $\kappa(x, y) = w_1\sin(x+y) + w_2\cos(x+y) + w_3\sin(2(x+y)) +  w_4\cos(2(x+y))$+4.1, where $w_i\sim \text{Uniform}(-1, 1)$.
We present two realizations in Figure \ref{fig_pm_kappa}.
\begin{figure}[H]
    \centering
    \includegraphics[scale = 0.4]{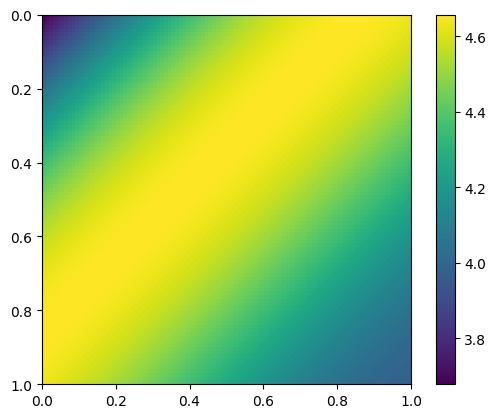}
    \includegraphics[scale = 0.4]{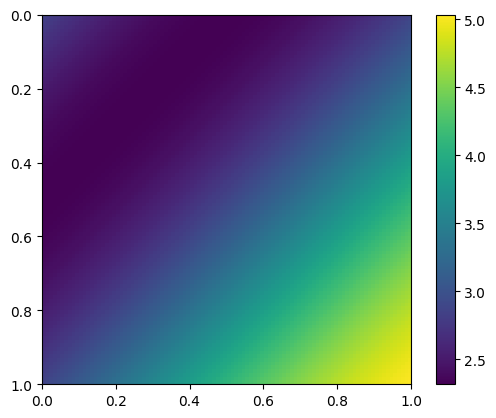}
    \caption{Demonstration of two permeability realizations for elliptic equation. Each permeability is determined by 4 degrees of freedom.}
    \label{fig_pm_kappa}
\end{figure}

\subsubsection{Results for elliptic equation}
\label{sec_result_elliptic}
In this section, we present the numerical results. 
The input observations (solutions of the equation) are uniformly sampled using a \(10 \times 10\) mesh, while the inverse QoIs are also discretized on a \(10 \times 10\) mesh. The prediction error is measured using the \(L_2\) norm, with the results illustrated in Figure \ref{fig_pm_100_original}. 
From the left panel of Figure \ref{fig_pm_100_original}, we observe that the relative \(L_2\) error decreases to \(0.7\%\) with \(360\) training samples and remains below \(1\%\) even with only \(108\) training samples. This demonstrates the accuracy of the network's predictions. 
Furthermore, as shown in the right panel of Figure \ref{fig_pm_100_original}, errors exhibit a linear decay in the logarithmic scale, which validates the theoretical results presented in Theorem \ref{thm.generalization}.
\begin{figure}[H]
    \centering
    \includegraphics[scale = 0.4]{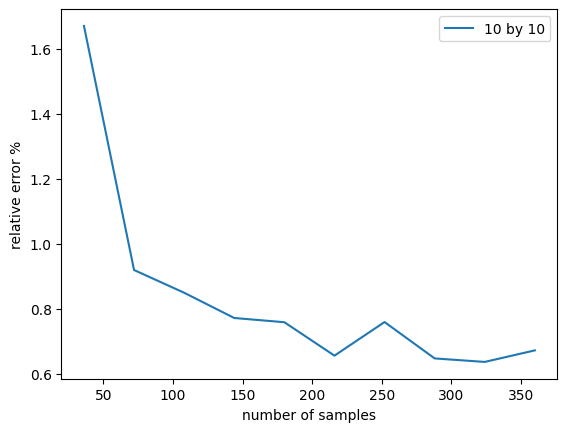}
    \includegraphics[scale = 0.4]{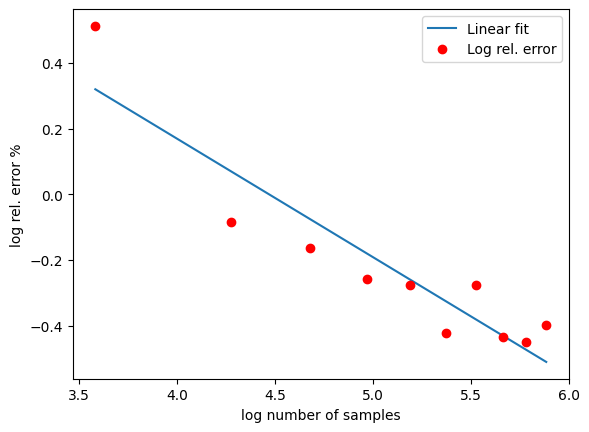}
    \caption{Elliptic example. Left: Error decay with respect to the number of training samples. Right: Error decay in
logarithmic scale. We infer the inverse QoIs, which correspond to the permeability discretized on
a $10\times 10$ mesh. To better illustrate the convergence rate, we fit a linear line to the error in the
logarithmic scale on the right.}
    \label{fig_pm_100_original}
\end{figure}

\subsubsection{Transfer to new QoIs (permeability on a finer mesh)}
We test the performance of C2BNet by modifying the QoIs while keeping the encoding part of the trained network fixed. This approach significantly reduces the computational cost of inferring new QoIs by avoiding the need to train an entirely new network.  
Specifically, the new inverse QoIs correspond to solutions defined on a finer $20 \times 20$ grid, which differs from the $10 \times 10$ training mesh evaluated in Section \ref{sec_result_elliptic}. 
To adapt the model, we retain all but the final layer and train only the last linear layer, which has a dimension of $12 \times 400$ ($d_{low} = 12$). The convergence of errors with respect to the number of training samples is presented in Figure \ref{fig_pm_transfer_400}.

\begin{figure}[H]
    \centering
    \includegraphics[scale = 0.4]{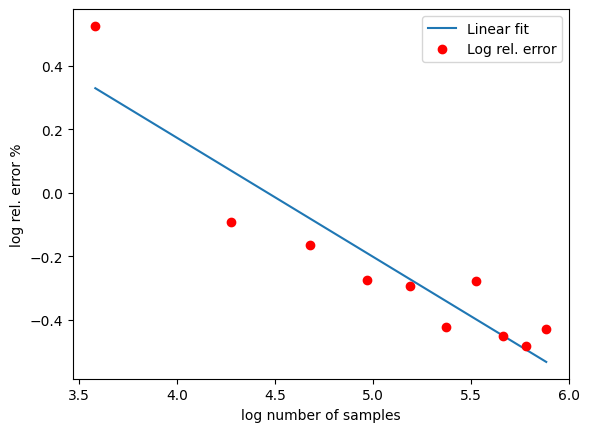}
    \includegraphics[scale = 0.4]{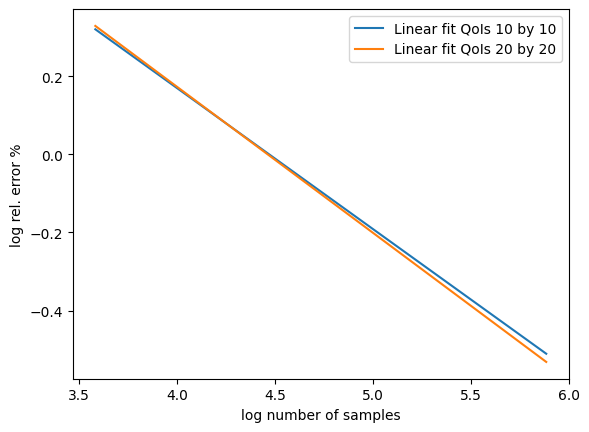}
    \caption{Elliptic example. Left: Error decay in logarithmic scale for the proposed method, which updates only the last linear layer for the downstream task with $20\times 20$-dimensional inverse QoIs. 
    Right: Convergence comparison between the pre-trained network for $10\times 10$-dimensional inverse QoIs and the new network.}
    \label{fig_pm_transfer_400}
\end{figure}

For comparison, we also retrain all layers of the network for the new task, where the QoIs are defined on a \(20 \times 20\) mesh. In particular, the full retraining process requires more cost than the proposed approach, which updates only the final layer.
Specifically, complete retraining requires updating a total of 36,312 parameters, while the proposed method only updates the last layer, which consists of 4,800 parameters.
However, the proposed method achieves relative errors similar to full retraining, demonstrating its effectiveness and accuracy; see Figure \ref{fig_pm_transfer_400_vs_retraining}.

\begin{figure}[H]
    \centering
    \includegraphics[scale = 0.4]{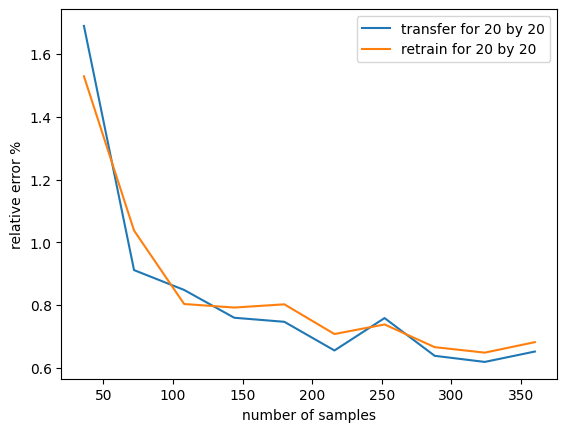}
    \includegraphics[scale = 0.4]{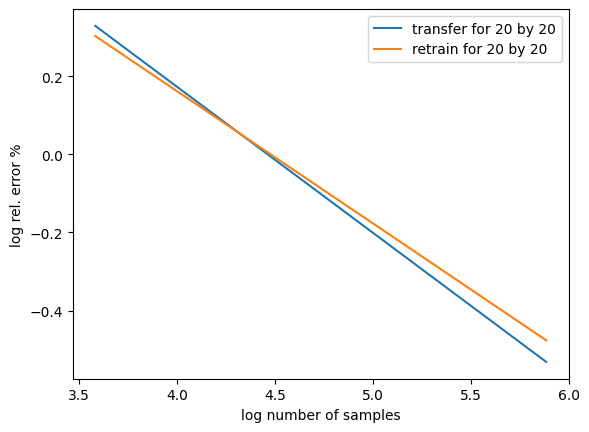}
    \caption{Time-dependent problem. Left: Error comparison between the proposed method, which transfers the trained weights
except for the last layer and updates only the last layer, and the approach of retraining all layers
for the downstream task with $20\times 20$ inverse QoIs. 
Notably, with significantly less trainable parameters, the proposed method achieves a similar prediction error as the full retraining. Right: error decay in log scale and linear fit the log error.}
    \label{fig_pm_transfer_400_vs_retraining}
\end{figure}

\subsection{Time-dependent diffusion equation}
In this example, we consider the parabolic equation,
\begin{align*}
    u_t - u_{xx}  = f, x\in[0, 2], t\in[0, 0.01],
\end{align*}
with Dirichlet boundary conditions. 
The goal is to find the initial condition given the solution at the terminal time. 
The initial condition used has the form: $u(x,0) = \sum_{i = 1}^3 w_i\sin(i\pi x) + p_i\cos(i\pi x)$, where $w_i, p_i$ are the weights uniformly sampled from $[-1, 1]$. We display 4 initial condition realizations in Figure \ref{fig_dr_ic}.

\begin{figure}[H]
    \centering
    \includegraphics[scale = 0.4]{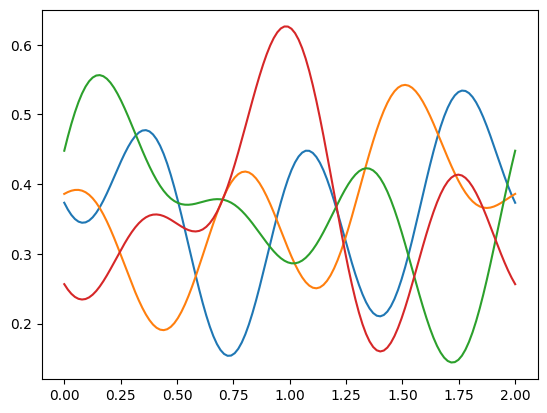}
    \caption{Demonstration of the four initial condition realizations for the time-dependent problems. Each realization is determined by six degrees of freedom.}
    \label{fig_dr_ic}
\end{figure}

\subsubsection{Results for the time-dependent diffusion equation}
The network input (observations) are solutions at $t = 0.01$, and we take $64$ points uniformly from $[0, 2]$, while the inverse QoIs are the initial conditions $u(x, 0)$ discretized by a $64$-equal-distance mesh.
The results are presented in Figure \ref{fig_dr64_original}. As illustrated in the left panel, the network achieves prediction relative errors below \(1\%\) when approximately 225 training samples are used. Furthermore, the error decay is plotted in logarithmic scale with respect to the number of training samples. The observed linear decay confirms the power-law behavior established in the theorem.

\begin{figure}[H]
    \centering
    \includegraphics[scale = 0.4]{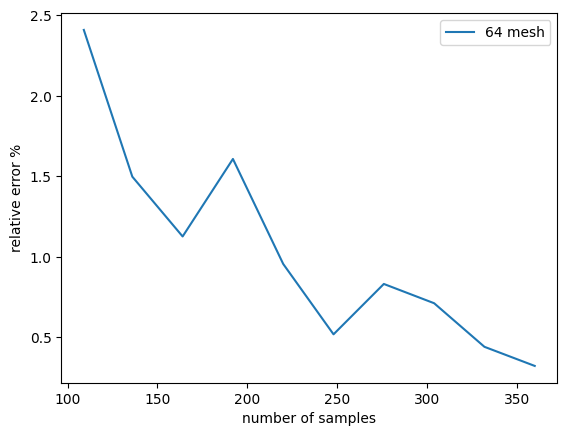}
    \includegraphics[scale = 0.4]{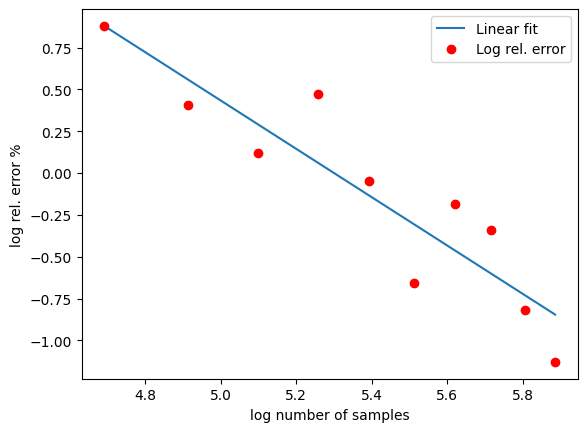}
    \caption{Left: Error decay with respect to the number of training samples. Right: Error decay in logarithmic scale. We infer the inverse QoIs, which correspond to initial conditions discretized on a 64-point mesh. To better illustrate the convergence rate, we fit a linear line to the error in the logarithmic scale on the right.}
    \label{fig_dr64_original}
\end{figure}

\subsubsection{Transfer to new QoIs (solutions on a finer mesh)}
To demonstrate that the latent layer (the second-to-last layer) provides valuable information, as established in the theorem, we consider a new set of QoIs corresponding to initial conditions defined on a denser mesh with 127 grid points.
The proposed method updates only the final linear layer of the network, which has a size of \(20 \times 127\) ($d_{low} = 20$), while keeping all preceding layers unchanged. This approach significantly reduces training costs compared to full retraining. The convergence results are presented in Figure \ref{fig_dr_transfer_127}.
From the left panel of the figure, we observe a linear decay when applying a logarithmic scale to both the error and the number of training samples. The right panel illustrates a close match in convergence rates between the pre-trained network, which predicts the original 64-point discretized initial condition, and the new network, where only a single linear layer is retrained for the updated QoIs.

\begin{figure}[H]
    \centering
    \includegraphics[scale = 0.4]{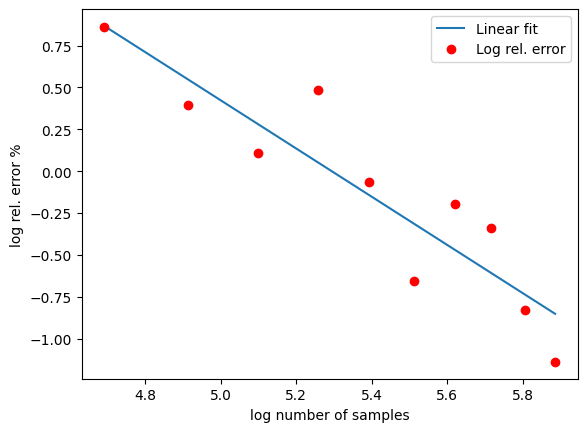}
    \includegraphics[scale = 0.4]{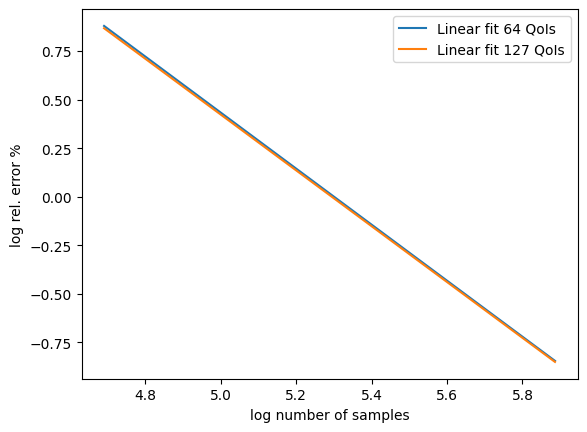}
    \caption{Left: Error decay in logarithmic scale for the proposed method, which updates only the last linear layer for the downstream task with 127-dimensional inverse QoIs. Right: Convergence comparison between the pre-trained network for 64-dimensional inverse QoIs and the new network.}
    \label{fig_dr_transfer_127}
\end{figure}

For comparison, we evaluate the proposed method against the approach of retraining all layers for the new tasks. The total number of trainable parameters in the full retraining method is 312,60, approximately 12.31 times larger than the proposed method with only 2,540 parameters, which only tunes the last linear layer. Despite significant savings in training cost, the accuracy of the proposed method remains unchanged. See Figure \ref{fig_dr_transfer_127vs_retraining} for the comparison.

\begin{figure}[H]
    \centering
    \includegraphics[scale = 0.4]{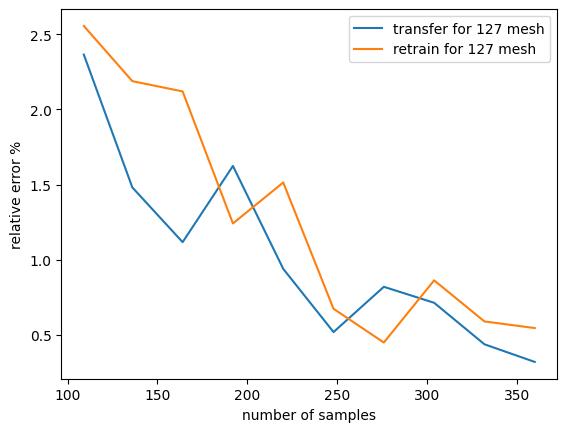}
    \includegraphics[scale = 0.4]{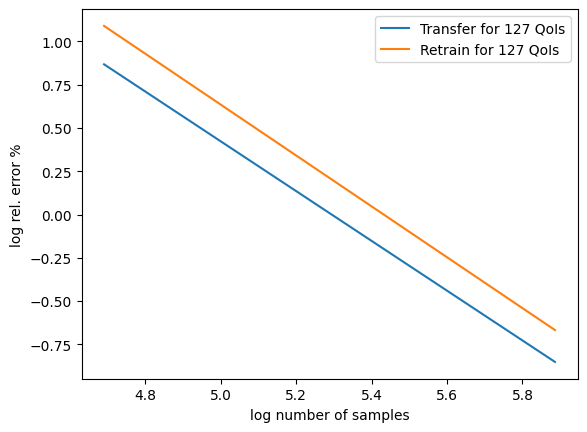}
    \caption{Time-dependent problem. Left: Error comparison between the proposed method, which transfers the trained weights except for the last layer and updates only the last layer, and the approach of retraining all layers for the downstream task with 127-dimensional inverse QoIs. 
    Notably, the full retraining involves approximately 12 times more trainable parameters than the proposed method; however, the prediction errors behave similarly. Right: error decay in log scale and linear fit the log error.}
    \label{fig_dr_transfer_127vs_retraining}
\end{figure}

\section{Conclusion}
\label{sec.conclusion}
In this work, we introduced C2BNet to solve inverse PDE problems within the operator learning paradigm. C2BNet consists of two components, a coefficient network and a linear basis network. C2BNet efficiently adapts to different discretizations through fine-tuning, leveraging a pretrained model to significantly reduce computational costs while maintaining high accuracy. We also established approximation and generalization error bounds for learning Lipschitz operators using C2BNet. Our theories show that C2BNet is adaptive to   low-dimensional data structures and achieves a fast convergence rate  depending on the intrinsic dimension of the dataset. The efficacy of C2BNet is demonstrated by systematic numerical experiments. The results confirm that C2BNet achieves a strong balance between computational efficiency and accuracy. These findings highlight the potential of C2BNet as a powerful and scalable tool to solve inverse problems in scientific computing and engineering applications.

\section*{Acknowledgment}
Hao Liu acknowledges the National Natural Science Foundation of China under the 12201530, and HKRGC ECS 22302123.
Wenjing Liao acknowledges the National Science Foundation under the NSF DMS 2145167 and the U.S. Department of Energy under the DOE SC0024348.
Guang Lin acknowledges the National Science Foundation under grants DMS-2053746, DMS-2134209, ECCS-2328241, CBET-2347401, and OAC-2311848. The U.S. Department of Energy also supports this work through the Office of Science Advanced Scientific Computing Research program (DE-SC0023161) and the Office of Fusion Energy Sciences (DE-SC0024583).

\bibliographystyle{ims}
\bibliography{ref}

\newpage
\appendix
\section*{Appendix}
\section{Proof of main results}
\label{sec.mainproof}
\subsection{Proof of Theorem \ref{thm.approx}}
\label{thm.approx.proof}
\begin{proof}[Proof of Theorem \ref{thm.approx}]
	We have
	\begin{align}
		S_{\cY}\circ \proj\circ\Psi(u)=&S_{\cY}\left(\sum_{k=1}^{d_2} \alpha_k^v\circ \Psi(u)\omega_k\right)=\sum_{k=1}^{d_2} \alpha_k^v\circ \Psi\circ D_{\cX}(\ub)S_{\cY}(\omega_k),
        \label{them.proof.error1}
	\end{align}
which is a linear combination of $S_{\cY}(\omega_k)$ with weight $\alpha_k^v\circ \Psi\circ D_{\cX}(\ub)$. Here we denote $\ub = S_{\cX}(u)$.

For the coefficients, we use the following lemma to show that each $\alpha_k^v\circ \Psi\circ D_{\cX}(\ub)$ is a Lipschitz function defined on $\cM$ (see a proof in Section \ref{lem.alpha.proof}).
\begin{lemma}\label{lem.alpha}
	Suppose Assumptions \ref{assum.cX}(iii), \ref{assum.cY} and \ref{assum.Psi} hold. For each $k$, 
	\begin{enumerate}[label=(\roman*)]
		\item $\sup_{v\in \cY}|\alpha_k^v(v)|\leq M_2|\Omega_{\cY}|$,
		\item  $\alpha_k^v\circ \Psi\circ D_{\cX}(\ub)$ is a Lipschitz function on $\cM$ with Lipschitz constant $C_{\cM}L_{\Psi} |\Omega_{\cX}|$, where $C_{\cM}$ is a constant depending on $\cM$.
	\end{enumerate}
\end{lemma}
We denote $\alpha^{\ub}_k(\ub)=\alpha_k^v\circ \Psi\circ D_{\cX}(\ub)$.
According to \citet[Theorem 3.1]{chen2022nonparametric}, for any $\varepsilon_1>0$ and for each $k$, there exists a network architecture $\cF_k=\cF_{\rm NN}(D_1,1,L_k,p_k,K_k,\kappa_k,R_k)$ that gives rise to $\widetilde{\alpha}^{\ub}_k\in \cF_k$ such that
$$
\|\widetilde{\alpha}^{\ub}_k-\alpha^{\ub}_k\|_{L^{\infty}(\cM)}<\varepsilon_1.
$$
Such a network architecture has 
\begin{align*}
	&L_k=O(\log\frac{1}{\varepsilon_1}), \ p=O(\varepsilon_1^{-d}),\ K_k=O(\varepsilon_1^{-d}\log \frac{1}{\varepsilon} + D_1\log \frac{1}{\varepsilon_1} +D_1\log D_1),\\
	&\kappa_k=O(\varepsilon_1^{-1}), \ R_k=M_2|\Omega_{\cY}|,
\end{align*}
where the constant hidden in $O$ depends on $d_1,M_1,M_2, |\Omega_{\cX}|,|\Omega_{\cY}|,L_{\Psi}, \tau$ and the surface area of $\cM$.

Denote 
$$
L'=\max_k L_k, \ p'=\max_k p_k, \ K'=\max_k K_k, \ \kappa'=\max_k \kappa_k, \ R'=\max_k R_k.
$$ 
There exists a network architecture $\cF_{coef}=\cF_{\rm NN}(L',d_2p',d_2K',\kappa', R')$ giving rise to $\widetilde{\balpha}=[\widetilde{\alpha}_1,...,\widetilde{\alpha}_{d_2}]^{\top}\in \cF_{coef}$ such that
\begin{align}
	\sup_{\ub\in \cM}\|\widetilde{\balpha}(\ub)-\balpha^{\ub}(\ub)\|_{\infty}< \varepsilon_1,
	\label{eq.approx.alpha}
\end{align}
where we denote $\balpha^{\ub}=[\alpha_1^{\ub},...,\alpha^{\ub}_{d_2}]^{\top}$.

Now we consider the bases in (\ref{them.proof.error1}). Since $\{\omega_k\}_{k=1}^{d_2}$ is independent of $u$, for any given discretization grids $S_{\cY}$, $\{S_{\cY}(\omega_k)\}_{k=1}^{d_2}$ is a set of fixed vectors, which can be realized by a matrix of dimension $D_2\times d_2$. Specifically, let $\bar{W}=[S_{\cY}(\omega_1),...,S_{\cY}(\omega_{d_2})]\in \RR^{D_2\times d_2}$. We construct $\widetilde{\Psi}$ as
$$
\widetilde{\Psi}=\begin{bmatrix}
	\bar{W} & \mathbf{0}\\ \mathbf{0} & -\bar{W}
\end{bmatrix} \ReLU \left( \begin{bmatrix}
\widetilde{\balpha} \\ -\widetilde{\balpha}
\end{bmatrix} \right).
$$
We have 
\begin{align*}
	&\sup_{u\in \cX}\|S_{\cY}\circ\proj\circ\Psi(u)-\widetilde{\Psi}(\ub)\|_{\infty}\\
	= & \sup_{u\in \cX}\left\|S_{\cY}\left( \sum_{k=1}^{d_2}\alpha^{\ub}_k(\ub)\omega_k\right)-\begin{bmatrix}
		\bar{W} & \mathbf{0}\\ \mathbf{0} & -\bar{W}
	\end{bmatrix} \ReLU \left( \begin{bmatrix}
		\widetilde{\balpha}(\ub) \\ -\widetilde{\balpha}(\ub)
	\end{bmatrix} \right)\right\|_{\infty} \\
	=& \sup_{u\in \cX}\|\sum_{k=1}^{d_2}\alpha^{\ub}_k(\ub)S_{\cY}(\omega_k)-\bar{W}\widetilde{\balpha}^{\ub}(\ub)\|_{\infty} \\
	=& \sup_{u\in \cX}\|\bar{W}\balpha^{\ub}(\ub)-\bar{W}\widetilde{\balpha}^{\ub}(\ub)\|_{\infty} \\
	=& \sup_{u\in \cX}\|\bar{W}(\balpha^{\ub}(\ub)-\bar{W}\widetilde{\balpha}^{\ub}(\ub))\|_{\infty} \\
	\leq& d_2M_2\varepsilon_1 ,
\end{align*}
where the last inequality uses (\ref{eq.approx.alpha}) and the fact that the absolute value of every element of $\bar{W}$ is bounded by $M_2$. Setting $\varepsilon_1=\frac{\varepsilon}{d_2M_2}$ finishes the proof. 
\end{proof}

\subsection{Proof of Corollary \ref{coro.approx}}
\label{coro.approx.proof}
\begin{proof}[Proof of Corollary \ref{coro.approx}]
    Let $\widetilde{\Psi}$ be the network constructed in Theorem \ref{thm.approx}. According to Assumption \ref{assum.cY}, we have
    \begin{align}
        &\sup_{u\in \cX}\|S_{\cY}\circ \Psi(u)-\widetilde{\Psi}\circ S_{\cX}(u)\|_{\infty} \nonumber\\
        \leq & \sup_{u\in \cX}\|S_{\cY}\circ \Psi(u)-S_{\cY}\circ \proj\circ\Psi(u)\|_{\infty} + \sup_{u\in \cX}\|S_{\cY}\circ \proj\circ \Psi(u)-\widetilde{\Psi}\circ S_{\cX}(u)\|_{\infty} \nonumber\\
        \leq& \zeta + \varepsilon.
    \end{align}
\end{proof}
\subsection{Proof of Theorem \ref{thm.generalization}}
\label{thm.generalization.proof}
\begin{proof}[Proof of Theorem \ref{thm.generalization}]
	We decompose the error as 
	\begin{align}
		&\EE_{\cS}\EE_{u\sim \rho}  \left[\|\widehat{\Psi}(\ub)-S_{\cY}\circ\Psi(u)\|_{S_{\cY}}^2\right] \nonumber\\
		=& \underbrace{2\EE_{\cS}\left[\frac{1}{n}\sum_{i=1}^n \|\widehat{\Psi}(\ub_i)-S_{\cY}\circ\Psi(u_i)\|_{S_{\cY}}^2\right]}_{\rm T_1} \nonumber\\
		&+\underbrace{\EE_{u\sim \rho} \EE_{\cS} \left[\|\widehat{\Psi}(\ub)-S_{\cY}\circ\Psi(u)\|_{S_{\cY}}^2\right]-2\EE_{\cS}\left[\frac{1}{n}\sum_{i=1}^n \|\widehat{\Psi}(\ub_i)-S_{\cY}\circ\Psi(u_i)\|_{S_{\cY}}^2\right]}_{\rm T_2}.
	\end{align}
We next derive bounds for ${\rm T_1}$ and ${\rm T_2}$.

\noindent{\bf Bounding ${\rm T_1}$.}
We bound ${\rm T_1}$ as
\begin{align}
	{\rm T_1}=&2\EE_{\cS}\left[\frac{1}{n}\sum_{i=1}^n \|\widehat{\Psi}(\ub_i)-S_{\cY}\circ\Psi(u_i)-S_{\cY}(\xi_i)+S_{\cY}(\xi_i)\|_{S_{\cY}}^2\right]\nonumber\\
	=& 2\EE_{\cS}\left[\frac{1}{n}\sum_{i=1}^n \|\widehat{\Psi}(\ub_i)-S_{\cY}\circ\Psi(u_i)-S_{\cY}(\xi_i)\|_{S_{\cY}}^2\right]\nonumber\\
	&+4\EE_{\cS}\left[\frac{1}{n}\sum_{i=1}^n \langle\widehat{\Psi}(\ub_i)-S_{\cY}\circ\Psi(u_i)-S_{\cY}(\xi_i),S_{\cY}(\xi_i)\rangle_{S_{\cY}}\right] + 2\EE_{\cS}\left[\frac{1}{n}\sum_{i=1}^n \|S_{\cY}(\xi_i)\|_{S_{\cY}}^2\right]\nonumber\\
	=& 2\EE_{\cS}\left[\frac{1}{n}\sum_{i=1}^n \|\widehat{\Psi}(\ub_i)-\widehat{\vb}_i\|_{S_{\cY}}^2\right]\nonumber\\
	&+4\EE_{\cS}\left[\frac{1}{n}\sum_{i=1}^n \langle\widehat{\Psi}(\ub_i)-S_{\cY}\circ\Psi(u_i),S_{\cY}(\xi_i)\rangle_{S_{\cY}}\right] - 2\EE_{\cS}\left[\frac{1}{n}\sum_{i=1}^n \|S_{\cY}(\xi_i)\|_{S_{\cY}}^2\right]\nonumber\\
	=& 2\EE_{\cS}\left[\inf_{\Psi_{\rm NN}\in \cF}\left[\frac{1}{n}\sum_{i=1}^n \|\Psi_{\rm NN}(\ub_i)-\widehat{\vb}_i\|_{S_{\cY}}^2\right]\right]\nonumber\\
	&+4\EE_{\cS}\left[\frac{1}{n}\sum_{i=1}^n \langle\widehat{\Psi}(\ub_i)-S_{\cY}\circ\Psi(u_i),S_{\cY}(\xi_i)\rangle_{S_{\cY}}\right] - 2\EE_{\cS}\left[\frac{1}{n}\sum_{i=1}^n \|S_{\cY}(\xi_i)\|_{S_{\cY}}^2\right]\nonumber\\
	\leq & 2\inf_{\Psi_{\rm NN}\in \cF}\left[\EE_{\cS}\left[\frac{1}{n}\sum_{i=1}^n \|\Psi_{\rm NN}(\ub_i)-S_{\cY}\circ\Psi(u_i)-S_{\cY}(\xi_i)\|_{S_{\cY}}^2-\|S_{\cY}(\xi_i)\|_{S_{\cY}}^2\right]\right]\nonumber\\
	&+4\EE_{\cS}\left[\frac{1}{n}\sum_{i=1}^n \langle\widehat{\Psi}(\ub_i)-S_{\cY}\circ\Psi(u_i),S_{\cY}(\xi_i)\rangle_{S_{\cY}}\right]\nonumber\\
	=& 2\inf_{\Psi_{\rm NN}\in \cF}\left[\EE_{\cS}\left[\frac{1}{n}\sum_{i=1}^n \|\Psi_{\rm NN}(\ub_i)-S_{\cY}\circ\Psi(u_i)\|_{S_{\cY}}^2\right]\right]
	+4\EE_{\cS}\left[\frac{1}{n}\sum_{i=1}^n \langle\widehat{\Psi}(\ub_i)-S_{\cY}\circ\Psi(u_i),S_{\cY}(\xi_i)\rangle_{S_{\cY}}\right]\nonumber\\
	=& 2\inf_{\Psi_{\rm NN}\in \cF}\left[\EE_{u\sim \rho}\left[\|\Psi_{\rm NN}(\ub)-S_{\cY}\circ\Psi(u)\|_{S_{\cY}}^2\right]\right]
	+4\EE_{\cS}\left[\frac{1}{n}\sum_{i=1}^n \langle\widehat{\Psi}(\ub_i)-S_{\cY}\circ\Psi(u_i),S_{\cY}(\xi_i)\rangle_{S_{\cY}}\right].
	\label{eq.gene.proof.T1}
\end{align}
For the first term, we have
\begin{align}
    &2\inf_{\Psi_{\rm NN}\in \cF}\left[\EE_{u\sim \rho}\left[\|\Psi_{\rm NN}(\ub)-S_{\cY}\circ\Psi(u)\|_{S_{\cY}}^2\right]\right] \nonumber\\
    \leq & 2\inf_{\Psi_{\rm NN}\in \cF}\left[\EE_{u\sim \rho}\left[\|2\Psi_{\rm NN}(\ub)-S_{\cY}\circ\proj\circ\Psi(u)\|_{S_{\cY}}^2+ 2\|S_{\cY}\circ\proj\circ\Psi(u)-S_{\cY}\circ\Psi(u)\|_{S_{\cY}}^2\right]\right]\nonumber\\
    \leq& 4\inf_{\Psi_{\rm NN}\in \cF}\left[\EE_{u\sim \rho}\left[\|2\Psi_{\rm NN}(\ub)-S_{\cY}\circ\proj\circ\Psi(u)\|_{S_{\cY}}^2\right]\right]+4\zeta^2.
\end{align}
For any $\varepsilon_1>0$, Theorem \ref{thm.approx} shows that there exists network architectures $\cF_{coef}=\cF_{\rm NN}(D_1,d_2,L,p,K,\kappa,R)$ \\$ \cF_{basis}=\cF_{\rm NN}(d_2,D_2,1,D_2,d_2D_2,M_2,M_2)$ and $\cF$ defined in (\ref{eq.cF}) giving rise to a $\widetilde{\Psi}\in \cF$ such that
\begin{align}
	\sup_{u\in \cX}\|S_{\cY}\circ\proj\circ\Psi(u)-\widetilde{\Psi}(\ub)\|_{\infty}\leq \varepsilon_1.
\end{align}
The network architecture $\cF_{coef}$ has
\begin{align}
	&L=O(\log\frac{1}{\varepsilon_1}), \ p=O(\varepsilon_1^{-d_1}),\ K=O(\varepsilon_1^{-d_1}\log \frac{1}{\varepsilon_1} + D_1\log \frac{1}{\varepsilon_1} +D_1\log D_1+D_2), \nonumber\\
	&\kappa=O(\varepsilon_1^{-1}), \ R=M_2|\Omega_{\cY}|.
	\label{eq.gene.proof.approx}
\end{align}
The constant hidden in $O$ depends on $d_1,d_2,M_1,M_2, |\Omega_{\cX}|,|\Omega_{\cY}|,L_{\Psi}, \tau$ and the surface area of $\cM$.
Thus the first term in (\ref{eq.gene.proof.T1}) is bounded by 
\begin{align}
	2\inf_{\Psi_{\rm NN}\in \cF}\left[\EE_{u\sim \rho}\left[\frac{1}{D_2}\|\Psi_{\rm NN}(\ub)-S_{\cY}\circ\Psi(u)\|_2^2\right]\right]\leq 4\varepsilon_1^2+ 4\zeta^2.
	\label{eq.gene.proof.T1.1}
\end{align}

To bound the second term, we can use \cite[Lemma 10]{liu2025generalization}: 
\begin{lemma}
	Under the condition of Theorem \ref{thm.generalization}, for any $\delta>0$, we have
	\begin{align}
		&\EE_{\cS}\left[\frac{1}{n}\sum_{i=1}^n \langle\widehat{\Psi}(\ub_i)-S_{\cY}\circ\Psi(u_i),S_{\cY}(\xi_i)\rangle_{S_{\cY}}\right] \nonumber\\
		&\leq 2\sigma \left( \sqrt{\EE_{\cS}[ \|\widehat{\Psi}-S_{\cY}\circ\Psi\|_n^2]} + \delta \right) \sqrt{\frac{4\log \cN(\delta,\cF,\|\cdot\|_{L^{\infty,\infty}})+6}{n}} + \delta\sigma,
		\label{eq.gene.proof.T1.2}
	\end{align}
where
$$
\|\widehat{\Psi}-S_{\cY}\circ\Psi\|_n^2=\frac{1}{n} \sum_{i=1}^n \|\widehat{\Psi}(\ub_i)-S_{\cY}\circ\Psi(u_i)\|_{S_{\cY}}^2.
$$
\end{lemma}
Substituting (\ref{eq.gene.proof.T1.1}) and (\ref{eq.gene.proof.T1.2}) into (\ref{eq.gene.proof.T1}) gives rise to
\begin{align}
	{\rm T_1}=&2\EE_{\cS}\left[ \|\widehat{\Psi}-S_{\cY}\circ\Psi\|_n^2\right] \nonumber\\
	\leq & 2\varepsilon_1^2+ 8\sigma \left( \sqrt{\EE_{\cS}[ \|\widehat{\Psi}-S_{\cY}\circ\Psi\|_n^2]} + \delta \right) \sqrt{\frac{4\log \cN(\delta,\cF,\|\cdot\|_{L^{\infty,\infty}})+6}{n}} + 4\delta\sigma.
	\label{eq.gene.proof.T1.3}
\end{align}
Denote 
\begin{align*}
	&s=\sqrt{\EE_{\cS}[ \|\widehat{\Psi}-S_{\cY}\circ\Psi\|_n^2]}\\
	&a=\varepsilon_1^2+ 4\sigma\delta\sqrt{\frac{4\log \cN(\delta,\cF,\|\cdot\|_{L^{\infty,\infty}})+6}{n}}+ 2\delta\sigma,\\
	&b=2\sigma\sqrt{\frac{4\log \cN(\delta,\cF,\|\cdot\|_{L^{\infty,\infty}})+6}{n}}.
\end{align*}
Relation (\ref{eq.gene.proof.T1.3}) can be rewritten as
\begin{align*}
	s^2\leq a+2bs,
\end{align*}
which implies
\begin{align*}
	&(s-b)^2\leq \sqrt{a}+b \\
	\Rightarrow &|s-b|\leq \sqrt{a+b^2} \leq \sqrt{a}+b.
\end{align*}
When $s\geq b$, we have
\begin{align*}
	&s-b\leq \sqrt{a}+b\\
	\Rightarrow & s\leq \sqrt{a}+2b\\
	\Rightarrow & s^2\leq (\sqrt{a}+2b)^2\leq 2a+8b^2.
\end{align*}
When $s<b$, the relation $s^2\leq 2a+8b^2$ is also true. Substituting the expression of $s,a,b$ into the relation, we get
\begin{align}
	{\rm T_1}=2s^2 \leq &8\varepsilon_1^2 + 8\zeta^2+ 16\sigma\delta\sqrt{\frac{4\log \cN(\delta,\cF,\|\cdot\|_{L^{\infty,\infty}})+6}{n}}+ 8\delta\sigma \nonumber\\
	&+ 128\sigma^2\frac{2\log \cN(\delta,\cF,\|\cdot\|_{L^{\infty,\infty}})+3}{n}.
\end{align}

\noindent{\bf Bounding ${\rm T_2}$.} To bound ${\rm T_2}$, we use \cite[Lemma 11]{liu2025generalization}:
\begin{lemma}[Lemma 11 of \cite{liu2025generalization}]
	Under the condition of Theorem \ref{thm.generalization}, for any $\delta>0$, we have
	\begin{align}
		{\rm T_2}\leq \frac{48M_2^2}{n}\log \cN\left( \frac{\delta}{4M_2},\cF,\|\cdot\|_{L^{\infty,\infty}}\right) + 6\delta.
	\end{align}
\end{lemma}

\noindent{\bf Putting ${\rm T_1}$ and ${\rm T_2}$ together.} Combining the error bound of ${\rm T_1}$ and ${\rm T_2}$ gives rise to
\begin{align}
	&\EE_{\cS}\EE_{u\sim \rho}  \left[\|\widehat{\Psi}(\ub)-S_{\cY}\circ\Psi(u)\|_{S_{\cY}}^2\right] \nonumber\\
	\leq &8\varepsilon_1^2+ 8\zeta^2+ 16\sigma\delta\sqrt{\frac{4\log \cN(\delta,\cF,\|\cdot\|_{L^{\infty,\infty}})+6}{n}}+ 8\delta\sigma \nonumber\\
	&+ 128\sigma^2\frac{2\log \cN(\delta,\cF,\|\cdot\|_{L^{\infty,\infty}})+3}{n} +\frac{48M_2^2}{n}\log \cN\left( \frac{\delta}{4M_2},\cF,\|\cdot\|_{L^{\infty,\infty}}\right) + 6\delta.
	\label{eq.gene.proof.1}
\end{align}
The covering number of a network class can be bounded by the following lemma:
\begin{lemma}[Lemma 5.3 of \cite{chen2022nonparametric}]
	\label{lem.covering}
	Let $\cF_{\rm NN}(D_1,D_2,L,p,K,\kappa,R)$ be a network architecture from $[-B_1,B_1]^{D_1}$ to $[-B_2,B_2]^{D_2}$ for some $B_1,B_2>0$. For any $\delta>0$, we have
	\begin{align}
		\cN(\delta,\cF_{\rm NN},\|\cdot\|_{L^{\infty,\infty}})\leq \left( \frac{2L^2(pB_1+2)\kappa^Lp^{L+1}}{\delta} \right)^{K}.
	\end{align}
\end{lemma}
Substituting the network architecture in (\ref{eq.gene.proof.approx}) into Lemma \ref{lem.covering} gives rise to 
\begin{align}
	\log \cN(\delta,\cF,\|\cdot\|_{L^{\infty,\infty}})\leq C_1D_1(\log D_1)\varepsilon_1^{-d_1} \log^2 \frac{1}{\varepsilon_1}\left(\log \frac{1}{\varepsilon_1} + \log \frac{1}{\delta}\right)
	\label{eq.gene.proof.logcn}
\end{align}
for some constant $C_1$ depending on $d_1,d_2,M_1,M_2, |\Omega_{\cX}|,|\Omega_{\cY}|,L_{\Psi}, L_D,\tau$ and the surface area of $\cM$. Setting $\delta=1/n$ and substituting (\ref{eq.gene.proof.logcn}) into (\ref{eq.gene.proof.1}) gives rise to
\begin{align}
	&\EE_{\cS}\EE_{u\sim \rho}  \left[\|\widehat{\Psi}(\ub)-S_{\cY}\circ\Psi(u)|_{S_{\cY}}^2\right]\leq  8\varepsilon_1^2+ 8\zeta^2+\frac{C_2D_1(\log D_1)}{n}\varepsilon_1^{-d_1} \log^2 \frac{1}{\varepsilon_1}\left(\log \frac{1}{\varepsilon_1} + \log n\right)
\end{align}
where $C_2$ is a constant depending on $d_1,d_2,M_1,M_2, |\Omega_{\cX}|,|\Omega_{\cY}|,L_{\Psi}, \sigma,\tau$ and the surface area of $\cM$. Setting $\varepsilon_1=n^{-\frac{2}{2+d_1}}$ finishes the proof. The network architecture is specified in (\ref{eq.gene.architecture}).
\end{proof}

\section{Proof of Lemma \ref{lem.alpha}}
\label{lem.alpha.proof}
\begin{proof}[Proof of Lemma \ref{lem.alpha}]	
	For (i), we have
	\begin{align*}
		|\alpha_k^v(v)|=&\langle v, \omega_k\rangle_{\cY}
		\leq  \|v\|_{\cY}\|\omega_k\|_{\cY}
		= \|v\|_{\cY}
		\leq  |\Omega_{\cY}|\|v\|_{L^{\infty}(\Omega_{\cY})}
		\leq  M_2|\Omega_{\cY}|.
	\end{align*}

	To prove (ii), for simplicity, we denote $\alpha^{\ub}_k(\ub)=\alpha_k^v\circ \Psi\circ D_{\cX}(\ub)$.
	Let $(U,\phi)$ be a chart of $\cM$ such that $\phi$ is Lipscthiz with Lipschitz constant $L_{\phi}$.  We need to show that $\alpha^{\ub}_k\circ \phi^{-1}: \phi(U)\rightarrow \RR$ is a Lipschitz function. 
	For any $\zb_1,\zb_2\in \phi(U)$, denote $\ub_j=\phi^{-1}(\zb_j),\ u_j=D_{\cX}(\ub_j),\ v_j=\Psi(u_j)$ for $j=1,2$.
	We have
	\begin{align}
		&|\alpha^{\ub}_k\circ \phi^{-1}(\zb_1)-\alpha^{\ub}_k\circ \phi^{-1}(\zb_2)| \nonumber\\
		=&|\alpha_k^v(v_1) -\alpha_k^v(v_2)|\nonumber\\
		= &|\langle v_1,\phi_k\rangle_{\cY} -\langle v_2,\phi_k\rangle_{\cY}|\nonumber\\
		= & |\langle v_1-v_2,\phi_k\rangle_{\cY}|\nonumber\\
		\leq & \|v_1-v_2\|_{\cY}\nonumber\\
		=&\|\Psi(u_1)-\Psi(u_2)\|_{\cY}\nonumber\\
		\leq &L_{\Psi}\|u_1-u_2\|_{\cX}\nonumber\\
		=& L_{\Psi} |\Omega_{\cX}|\|u_1-u_2\|_{L^{\infty}(\Omega_{\cX})}\nonumber\\
		=&L_{\Psi} |\Omega_{\cX}||D_{\cX}\circ \phi^{-1}(\zb_1)-D_{\cX}\circ \phi^{-1}(\zb_2)|,
		\label{eq.alphaLip.1}
	\end{align}
 
	According to Assumption \ref{assum.cX}(iii), $D_{\cX}\circ \phi^{-1}$ is a Lipschitz function. For any given finite atlas of $\cM$, there exists a constant $C_{\cM}$ that is an upper bound of the Lipschitz constant of $D_{\cX}\circ \phi^{-1}$ for all $\phi$'s in the atlas. Applying this property to (\ref{eq.alphaLip.1}) gives rise to
	\begin{align}
		|\alpha^{v}_k(v_1)-\alpha^{v}_k(v_2)|\leq CL_{\Psi} |\Omega_{\cX}|\|\zb_1-\zb_2\|_2.
	\end{align}
\end{proof}

\end{document}